\documentclass[a4paper, 12pt]{article}

\usepackage[utf8]{inputenc}
\usepackage[T1]{fontenc}
\usepackage[margin=1in]{geometry} 
\usepackage{amsmath, amssymb, amsthm} 
\usepackage{graphicx} 
\usepackage{booktabs} 
\usepackage{algorithm} 
\usepackage{algpseudocode} 
\usepackage{hyperref} 
\usepackage{xcolor} 
\usepackage{subcaption} 
\usepackage{bm} 
\usepackage{seqsplit} 
\usepackage{float}

\newcommand{\loss}{\mathcal{L}}

\newcommand{\epsc}{\varepsilon_c}
\newcommand{\dist}{d}
\newtheorem{definition}{Definition}
\newtheorem{theorem}{Theorem}

\newtheorem{proposition}{Proposition}
\newtheorem{remark}{Remark} 
\newtheorem{assumption}{Assumption}

\title{\textbf{ Manifold Percolation: from generative model to Reinforce learning}\\Metrics, Dynamics, and Topological Supervision}
\author{Rui Tong \\ \textit{Department of Statistics}\\ \textit{University of Warwick}\\ \textit{Rui.Tong@Warwick.ac.uk}}

\begin{document}
\maketitle

\begin{abstract}
Generative modeling is typically framed as learning mapping rules, but from an observer's perspective without access to these rules, the task manifests as disentangling the geometric \textbf{support} from the probability \textbf{distribution}.
We propose that \textbf{Continuum Percolation} is uniquely suited for this support analysis, as the sampling process effectively projects high-dimensional density estimation onto a geometric counting problem on the support.
In this work, we establish a rigorous isomorphism between the topological phase transitions of Random Geometric Graphs and the underlying data manifold in high-dimensional space.
By analyzing the relationship between our proposed Percolation Shift metric and FID, we demonstrate that our metric captures structural pathologies (such as implicit mode collapse) where statistical metrics fail.
Finally, we translate this topological phenomenon into a differentiable loss function to guide training.
Experimental results confirm that this approach not only prevents manifold shrinkage but fosters a \textbf{"Synergistic Improvement,"} where topological stability acts as a prerequisite for sustained high fidelity in both static generation and sequential decision-making.
\end{abstract}

\section{Introduction}

\subsection{Motivation: The "Phantom Manifold" Thought Experiment}
\label{sec:thought_experiment}

Why do we need a topological "Observer's View" for generative modeling? Consider the following thought experiment.

Let $\mathcal{D}_{\text{noise}} = \{ \bm{z}^{(k)} \}_{k=1}^N$ be a dataset of $N$ samples drawn from pure white noise, $\bm{z} \sim \mathcal{N}(0, \mathbf{I}_D)$. By definition, the pixels are independent: the ground-truth Jacobian is strictly diagonal, $J_{ij} = \frac{\partial z_i}{\partial z_j} = 0$ for $i \neq j$. The true data manifold is the entire ambient space $\mathbb{R}^D$ (or a high-dimensional ball), with no low-dimensional structure.

Now, suppose we train a powerful neural network (e.g., a diffusion model\cite{sohl2015thermo, ho2020ddpm}) to model this dataset. Due to the finite $N$ and the over-parameterized nature of deep networks, the model will inevitably \textit{overfit}. It will memorize the specific noise patterns in $\mathcal{D}_{\text{noise}}$.
Mathematically, the learned function $f_\theta$ will encode spurious correlations, resulting in a non-zero sensitivity map:
\begin{equation}
    \left| \frac{\partial \hat{x}_i}{\partial \hat{x}_j} \right| > 0 \quad \text{for } i \neq j.
\end{equation}
To the model, the data appears to lie on a complex, low-dimensional \textbf{"Phantom Manifold"} defined by these spurious dependencies.

Here lies the paradox:
\begin{itemize}
    \item \textbf{From the Model's Internal View (Loss):} The training loss approaches zero. The model believes it has perfectly learned the "structure" of the data.
    \item \textbf{From the Physical View (Independence):} The model is hallucinating structure where there is only entropy.
\end{itemize}

This paradox reveals that we cannot trust the model's internal proxies (loss, likelihood) to judge geometric validity. We require an external, observer-centric metric capable of distinguishing between a \textit{Phantom Manifold} (memorized noise clumps) and a \textit{True Manifold} (semantic structure). This motivates our use of \textbf{Continuum Percolation}: a topological tool specifically designed to analyze the connectivity and volume of point clouds without accessing the generating rules.

\section{Theoretical Framework: Manifold Percolation}
\label{sec: theory}
In this section, we establish the mathematical foundation of our framework. We bridge the gap between generative modeling and statistical physics by formalizing the "Observer's View" through the lens of Continuum Percolation Theory.

\subsection{Preliminaries: Random Geometric Graphs on Manifolds}
Let the high-dimensional data space be $\mathbb{R}^D$. We assume the
\textit{Manifold Hypothesis}, positing that the real data distribution
$p_{\text{data}}$ is supported on a lower-dimensional compact Riemannian
manifold $\mathcal{M} \subset \mathbb{R}^D$ with intrinsic dimension $d \ll D$
and total volume $\mathrm{Vol}(\mathcal{M})$.
Consider a generative model $p_\theta$ attempting to approximate $p_{\text{data}}$.
From an observer’s perspective, we access the manifold solely through finite
sampling. Let
\[
S_N = \{ x^{(i)} \}_{i=1}^N \overset{\text{i.i.d.}}{\sim} p_\theta
\]
be a set of $N$ generated samples.

\begin{definition}[Random Geometric Graph]
Following the classical \textit{Continuum Percolation} framework introduced by
Gilbert~\cite{gilbert1961} and formalized by Meester \& Roy
\cite{meester1996}, a Random Geometric Graph (RGG) $\mathcal{G}(S_N,
\varepsilon)$ is defined by a vertex set $V = S_N$ and an edge set
\begin{equation}
E_\varepsilon
= \{ (x^{(i)}, x^{(j)}) \mid \mathrm{dist}(x^{(i)}, x^{(j)}) \le \varepsilon,\ i \neq j \}.
\end{equation}
\end{definition}

Penrose~\cite{penrose2003} proved that as $N \to \infty$, the
connectivity and topological properties of $\mathcal{G}(S_N, \varepsilon)$
constructed on a Riemannian manifold converge (in Hausdorff distance and
component structure) to those of the underlying manifold $\mathcal{M}$ itself.
This provides the theoretical basis for using RGGs as manifold estimators.

\subsection{Phase Transition and the Order Parameter}
Percolation theory studies the emergence of global connectivity from local
interactions. We define the \textbf{Giant Component Ratio}
\begin{equation}
P_\infty(\varepsilon)
= \frac{|C_{\max}(\mathcal{G}_\varepsilon)|}{N},
\end{equation}
where $|C_{\max}|$ denotes the size of the largest connected component.
Since $P_\infty(\varepsilon)$ is non-decreasing, we define the finite-size
critical threshold
\begin{equation}
\varepsilon_c(N)
= \inf \{ \varepsilon \mid P_\infty(\varepsilon) \ge 0.5 \}.
\end{equation}

\paragraph{Justification for the 0.5 Threshold.}
The choice is grounded in the \textit{Uniqueness of the Infinite Cluster}
theorem by Burton \& Keane~\cite{burton1989}, which proves that in
continuum percolation, the giant component is almost surely unique; thus, two
components cannot each exceed $50\%$ of the total mass. Therefore, reaching
$P_\infty = 0.5$ necessarily marks the emergence of the unique dominant cluster.

\paragraph{Finite-Size Effects.}
As shown in classical percolation studies
(Stauffer \& Aharony~\cite{stauffer1994};
Bollobás \& Riordan~\cite{bollobas2006}),
the transition is a sharp step function as $N \to \infty$, but becomes a
sigmoid-like curve for finite systems. The point $\varepsilon_c(N)$ therefore
acts as a meaningful geometric descriptor of the connectivity structure.

\subsection{Scaling Laws and Intrinsic Dimension}
\label{sec:scaling_proof}
To justify $\varepsilon_c$ as a proxy for manifold volume, we first derive its scaling
under the uniform assumption.

\begin{proposition}[Finite-Size Scaling Law]
For $N$ samples uniformly distributed on a $d$-dimensional manifold
$\mathcal{M}$, the percolation threshold satisfies
\begin{equation}
\varepsilon_c(N) \propto N^{-1/d}.
\end{equation}
\end{proposition}

\begin{proof}
Penrose~\cite{penrose1997} established that
the connectivity threshold corresponds to the longest edge of the Minimal
Spanning Tree (MST), occurring when the expected degree reaches a critical
dimension-dependent constant $\lambda_c(d)$.
For sufficiently small $\varepsilon$, the volume of the geodesic ball satisfies
\[
\mathrm{Vol}(B_\varepsilon(x) \cap \mathcal{M})
= C_d \varepsilon^d + O(\varepsilon^{d+2}),
\]
where $C_d$ is the volume of the unit $d$-ball.
Given uniform density $\rho = N / \mathrm{Vol}(\mathcal{M})$, the expected degree is
\begin{equation}
\langle k \rangle
\approx \rho \cdot C_d \varepsilon^d
= \frac{N C_d \varepsilon^d}{\mathrm{Vol}(\mathcal{M})}.
\end{equation}
At the critical point $\varepsilon = \varepsilon_c$, we have
\begin{equation}
\lambda_c
= \frac{N C_d \varepsilon_c^d}{\mathrm{Vol}(\mathcal{M})}
\quad \Rightarrow \quad
\varepsilon_c^d
= \frac{\lambda_c \mathrm{Vol}(\mathcal{M})}{C_d}\, \frac{1}{N}.
\end{equation}
Taking the $d$-th root yields
\begin{equation}
\varepsilon_c(N)
= \left( \frac{\lambda_c \mathrm{Vol}(\mathcal{M})}{C_d} \right)^{\!1/d} N^{-1/d}.
\end{equation}
\end{proof}

\begin{figure}[H]
    \centering
    \includegraphics[width=0.8\linewidth]{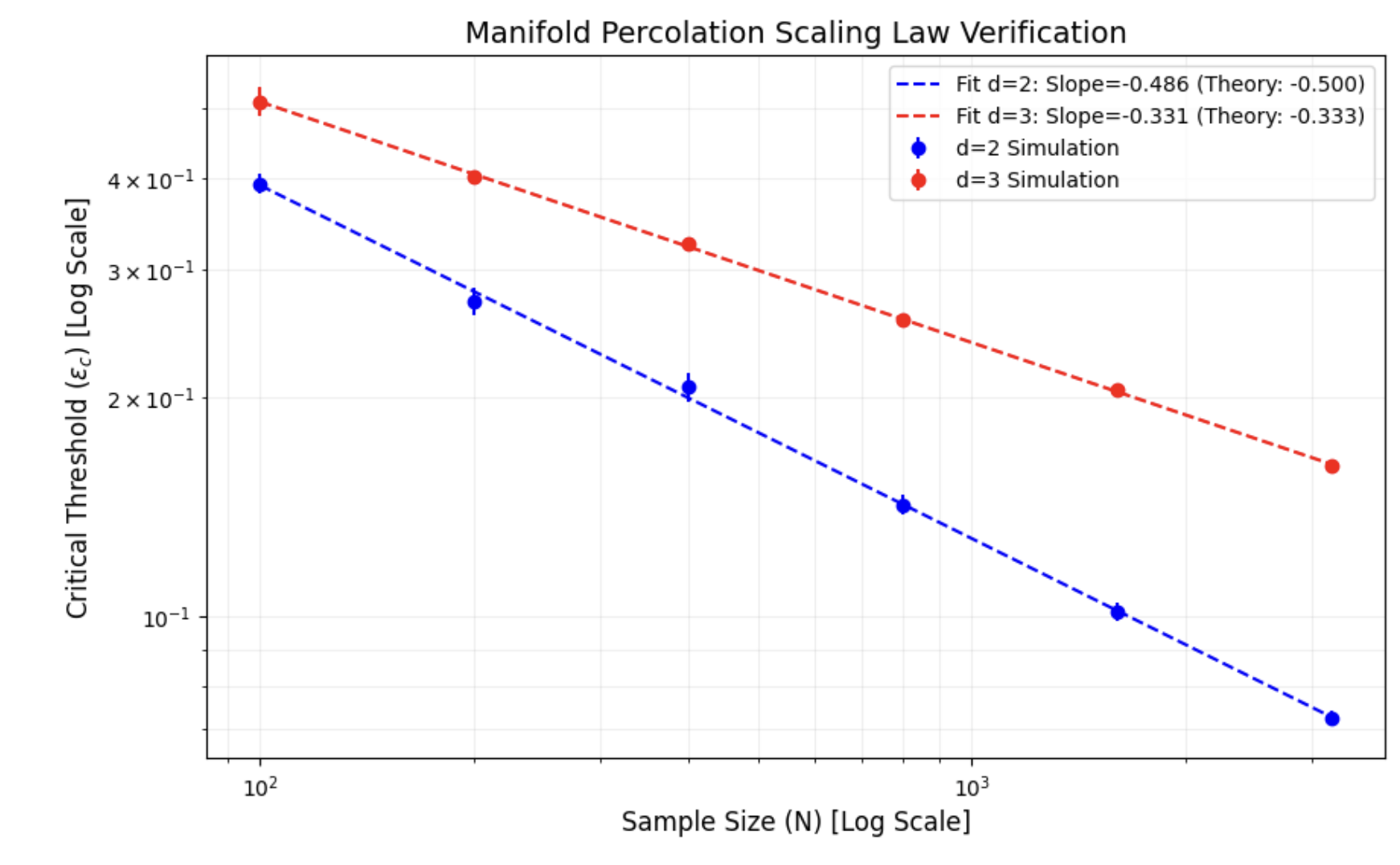}
    \caption{\textbf{Empirical Verification of Manifold Percolation Scaling.}
    We simulate Random Geometric Graphs on hyperspheres ($S^d$) with varying sample sizes $N \in [100, 3000]$.
    The log-log plot reveals a linear relationship where the slope corresponds precisely to $-1/d$, validating the theoretical scaling law $\varepsilon_c \propto N^{-1/d}$.
    Error bars denote the standard deviation over 5 independent trials, demonstrating the metric's stability.}
    \label{fig:scaling_law}
\end{figure}

We now generalize this result to realistic scenarios where the data distribution is non-uniform.

\begin{proposition}[Generalization to Non-Uniform Density]
Assume samples are drawn from a non-uniform probability density function $p(x)$ supported on $\mathcal{M}$. The critical threshold $\varepsilon_c$ scales as:
\begin{equation}
    \varepsilon_c(N) \propto \mathcal{H}_2(p)^{-1/d} \cdot N^{-1/d}
\end{equation}
where $\mathcal{H}_2(p) = \int_{\mathcal{M}} p^2(x) dx$ relates to the collision entropy. Crucially, the scaling exponent $-1/d$ remains invariant.
\end{proposition}

\begin{proof}
For a non-uniform density $p(x)$, the local expected degree at $x$ is $\mathbb{E}[k(x)] \approx N p(x) C_d \varepsilon^d$. The global connectivity transition is governed by the average degree of the system reaching a critical value $\lambda_c$:
\begin{equation}
    \langle k \rangle = \int_{\mathcal{M}} \mathbb{E}[k(x)] p(x) dx \approx N C_d \varepsilon^d \int_{\mathcal{M}} p^2(x) dx.
\end{equation}
Solving for $\varepsilon_c$ at $\langle k \rangle = \lambda_c$ yields $\varepsilon_c \propto (\int p^2(x) dx)^{-1/d} N^{-1/d}$. This confirms that the scaling law is robust to density variations, with the "effective volume" determined by the Rényi entropy.
\end{proof}

\subsection{Detecting Mode Collapse via Manifold Shrinkage}
Using the scaling relationship above, we derive the main theorem supporting our
\textbf{Percolation Shift} metric.

\begin{theorem}[Manifold Shrinkage Theorem]
\label{thm:shrinkage}
Let $\varepsilon_c^{\text{real}}$ and $\varepsilon_c^{\text{model}}$ denote the
percolation thresholds of real data $p_{\text{data}}$ and model distribution
$p_\theta$, respectively, computed with the same sample size $N$. If the model
exhibits \textbf{manifold shrinkage},
\[
\mathrm{Vol}(\mathcal{M}_{\text{model}})
<
\mathrm{Vol}(\mathcal{M}_{\text{real}}),
\]
and both share intrinsic dimension $d$, then
\begin{equation}
\varepsilon_c^{\text{model}} < \varepsilon_c^{\text{real}}.
\end{equation}
Thus $\Delta\varepsilon_c = \varepsilon_c^{\text{model}} -
\varepsilon_c^{\text{real}} < 0$ is a quantitative signature of mode collapse.
\end{theorem}

\begin{proof}
From Section~\ref{sec:scaling_proof},
\[
\varepsilon_c \propto (\mathrm{Vol}(\mathcal{M}))^{1/d}.
\]
Let $\alpha$ be the constant factor (which depends on $d$ and $\lambda_c$). Then
\[
\varepsilon_c^{\text{model}}
= \alpha (\mathrm{Vol}(\mathcal{M}_{\text{model}}))^{1/d},
\qquad
\varepsilon_c^{\text{real}}
= \alpha (\mathrm{Vol}(\mathcal{M}_{\text{real}}))^{1/d}.
\]
Since $f(x)=x^{1/d}$ is strictly increasing for $x>0$, the volume inequality
immediately yields the result.
\end{proof}

\subsection{Metric Spaces: From Pixel to Semantic Percolation}
\label{sec:metric_spaces}
While the standard RGG is constructed in the raw pixel space $\mathcal{X} \subset \mathbb{R}^{D}$, relying solely on Euclidean distance ($d_{\text{pix}}(x,y) = \|x-y\|_2$) can be misleading. Pixel-wise distance is known to be sensitive to translation and noise, often failing to capture high-level structural differences.

To rigorously evaluate the \textit{semantic} topology, we extend the percolation framework to the deep feature space defined by a pre-trained encoder $\phi: \mathcal{X} \to \mathbb{R}^{d'}$. We employ a VGG-16 \cite{simonyan2015vgg}network pre-trained on ImageNet as the feature extractor $\phi$. The \textbf{Semantic Percolation Process} is defined by the edge set:
\begin{equation}
E_\varepsilon^{\text{feat}} = \{ (x^{(i)}, x^{(j)}) \mid \|\phi(x^{(i)}) - \phi(x^{(j)})\|_2 \le \varepsilon \}.
\end{equation}
This projection allows us to detect "Semantic Collapse," where a model might maintain pixel variance (noise) but fails to cover the semantic diversity of the true manifold.

\section{Manifold Dynamics and Topological Diagnosis}
\label{sec:diagnosis}

We propose the \textbf{Percolation Shift} as a diagnostic metric to quantify the topological alignment between the generated distribution and the real data manifold:
\begin{equation}
    \Delta \varepsilon_c = \varepsilon_c(\text{Model}) - \varepsilon_c(\text{Real}).
\end{equation}
According to the Manifold Shrinkage Theorem (Theorem~\ref{thm:shrinkage}), a negative shift ($\Delta \varepsilon_c < 0$) serves as a \emph{monotone proxy} for manifold shrinkage, i.e., it provides strong evidence that
\[
    \mathrm{Vol}(\mathcal{M}_{\text{model}}) < \mathrm{Vol}(\mathcal{M}_{\text{real}}).
\]

\subsection{The Geometric Mechanism of Shrinkage}

To validate our theoretical framework, we first conduct a controlled experiment on a synthetic 2D manifold. We simulate \emph{Implicit Mode Collapse} by reducing the intra-mode variance of a Gaussian mixture while preserving the mode centers.
As predicted, the percolation phase transition curve shifts to the left, resulting in a negative $\Delta \varepsilon_c$. This confirms that our metric is sensitive to the effective volume of the support, effectively distinguishing between a dispersed (healthy) manifold and a contracted (collapsed) one.

\begin{figure}[H]
    \centering
    \includegraphics[width=0.95\linewidth]{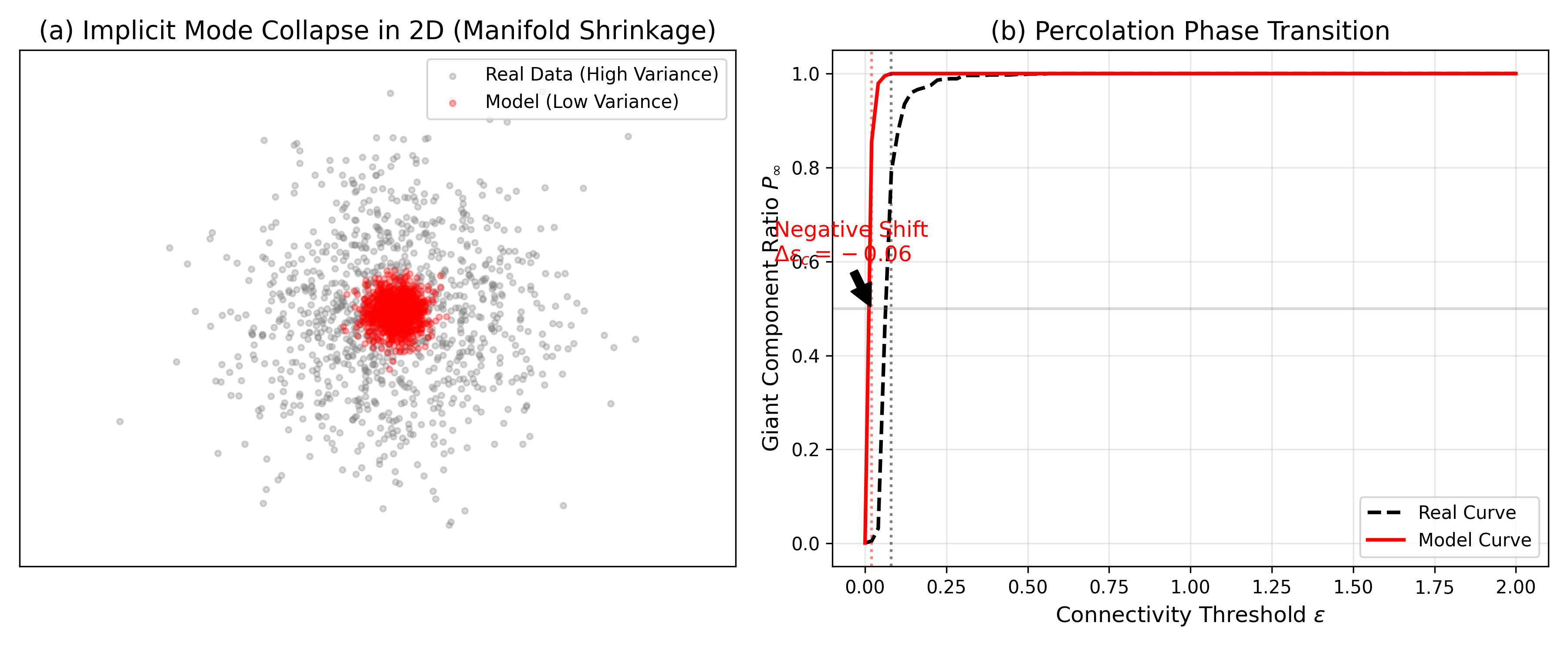}
    \caption{\textbf{Controlled Experiment on a Toy Manifold.} 
    (a) Simulating Implicit Mode Collapse by reducing variance (red) relative to ground truth (gray). 
    (b) The percolation curve shifts left ($\Delta \varepsilon_c < 0$), empirically validating the shrinkage behavior predicted by Theorem~\ref{thm:shrinkage}.}
    \label{fig:toy_example}
\end{figure}

\subsection{The Deceptive Nature of Visual Projections}
\label{sec:visual_deception}

A common practice in generative modeling is to rely on dimensionality reduction (e.g., UMAP) for visual inspection. However, by contrasting visual projections with quantitative percolation curves, we uncover a critical \textbf{Visual Illusion}.

As shown in \textbf{Figure \ref{fig:anatomy}a}, UMAP projections depict the generated samples converging towards the data manifold, creating an impression of structural stability from Epoch 60 to 300. The projection at Epoch 300 appears visually dense and well-structured, potentially leading observers to believe the model has converged to a healthy state.

\textbf{However, this visualization is geometrically unfaithful.} It fails to capture the drastic changes in effective support volume. The true topological dynamics are revealed in \textbf{Figure \ref{fig:anatomy}b (Percolation Curves)}:
\begin{itemize}
    \item \textbf{Initial State (Epoch 0):} The percolation curve is positioned far to the right of the real data baseline, corresponding to a large positive shift ($\Delta \varepsilon_c = +11.49$). This quantifies the initial high-entropy noise state (\emph{Zone III} in our phase diagram, i.e., an over-expanded manifold), whereas UMAP merely displays this as a diffuse cloud without indicating the scale of the volume excess.
    \item \textbf{Late-Stage Collapse (Epoch 300):} While UMAP suggests a robust manifold similar to earlier epochs, Fig.~\ref{fig:anatomy}b reveals that the percolation curve has shifted significantly to the left of the baseline ($\Delta \varepsilon_c = -1.12$). This indicates \textbf{Implicit Mode Collapse} (\emph{Zone I}: a contracted, low-volume manifold).
\end{itemize}

This discrepancy demonstrates that dimensionality reduction algorithms, by optimizing for local neighborhood preservation, tend to normalize density differences. They artificially ``inflate'' the collapsed, low-volume manifold (Ep 300) to fill the visualization space, rendering the shrinkage invisible to the naked eye.

\begin{remark}[The Ambiguity of Visual Clumping: Collapse vs. Divergence]
Our experiments further identify a paradoxical \emph{Visual Clump} phenomenon. 
When a model generates high-frequency noise or artifacts (Manifold Divergence in Zone III), $\Delta \varepsilon_c$ correctly exhibits a large \textbf{positive spike} ($\Delta \varepsilon_c \gg 0$). 
However, UMAP often projects these high-dimensional isotropic noise samples as a \textbf{tight, off-manifold clump}, visually indistinguishable from a collapsed mode.  
Thus, visual inspection is doubly deceptive: it conflates \textbf{Implicit Shrinkage} (Zone I) and \textbf{Manifold Explosion} (Zone III) into similar visual blobs. In contrast, our Percolation Shift cleanly separates them at the sign level in practice ($\Delta \varepsilon_c < 0$ vs.\ $\Delta \varepsilon_c > 0$), providing a robust diagnostic signal that visual projections lack.
\end{remark}

\begin{figure*}[t]
    \centering
    \begin{subfigure}[b]{\textwidth}
        \centering
        \includegraphics[width=\textwidth]{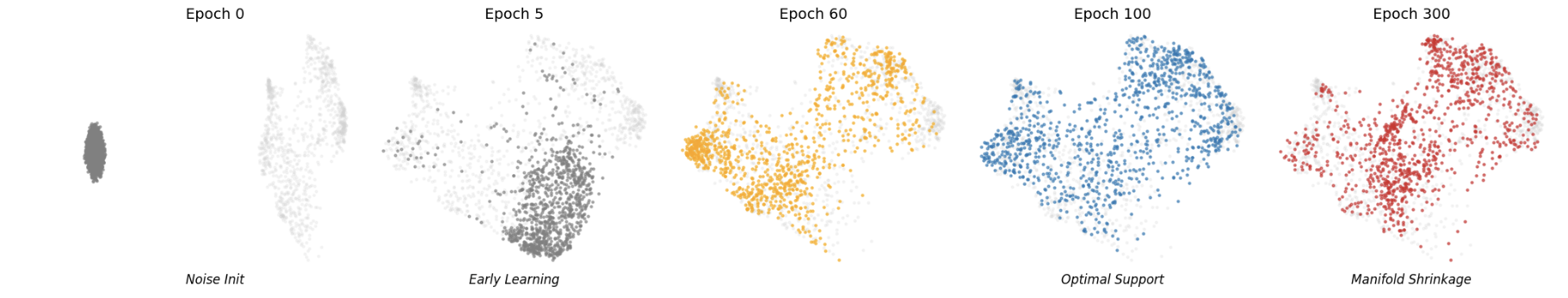}
        \caption{\textbf{The Illusion of Stability:} UMAP projections (top row) show the samples converging to the manifold location but mask the volume shrinkage. Visually, Ep 300 (red) appears as robust as Ep 100, contradicting the topological reality.}
        \label{fig:anatomy_umap}
    \end{subfigure}
    
    \vspace{0.2cm}
    
    \begin{subfigure}[b]{0.49\textwidth}
        \centering
        \includegraphics[width=\textwidth]{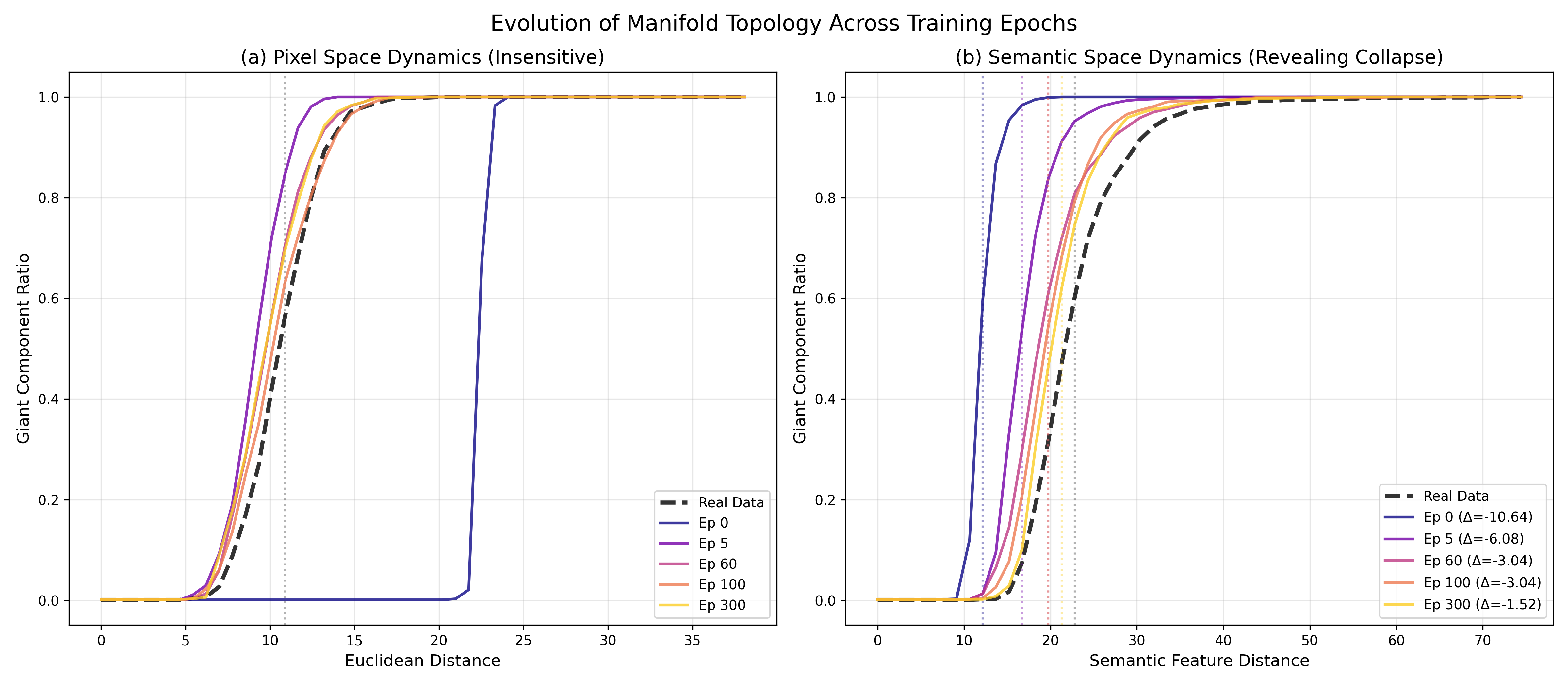}
        \caption{\textbf{The Topological Truth:} Pixel Space (left) overlaps perfectly (illusion). However, Semantic Space (right) reveals the true dynamics: from Over-expansion at Ep 0 (curve right) to Shrinkage at Ep 300 (curve left).}
        \label{fig:anatomy_perc}
    \end{subfigure}
    \hfill
    \begin{subfigure}[b]{0.49\textwidth}
        \centering
        \includegraphics[width=\textwidth]{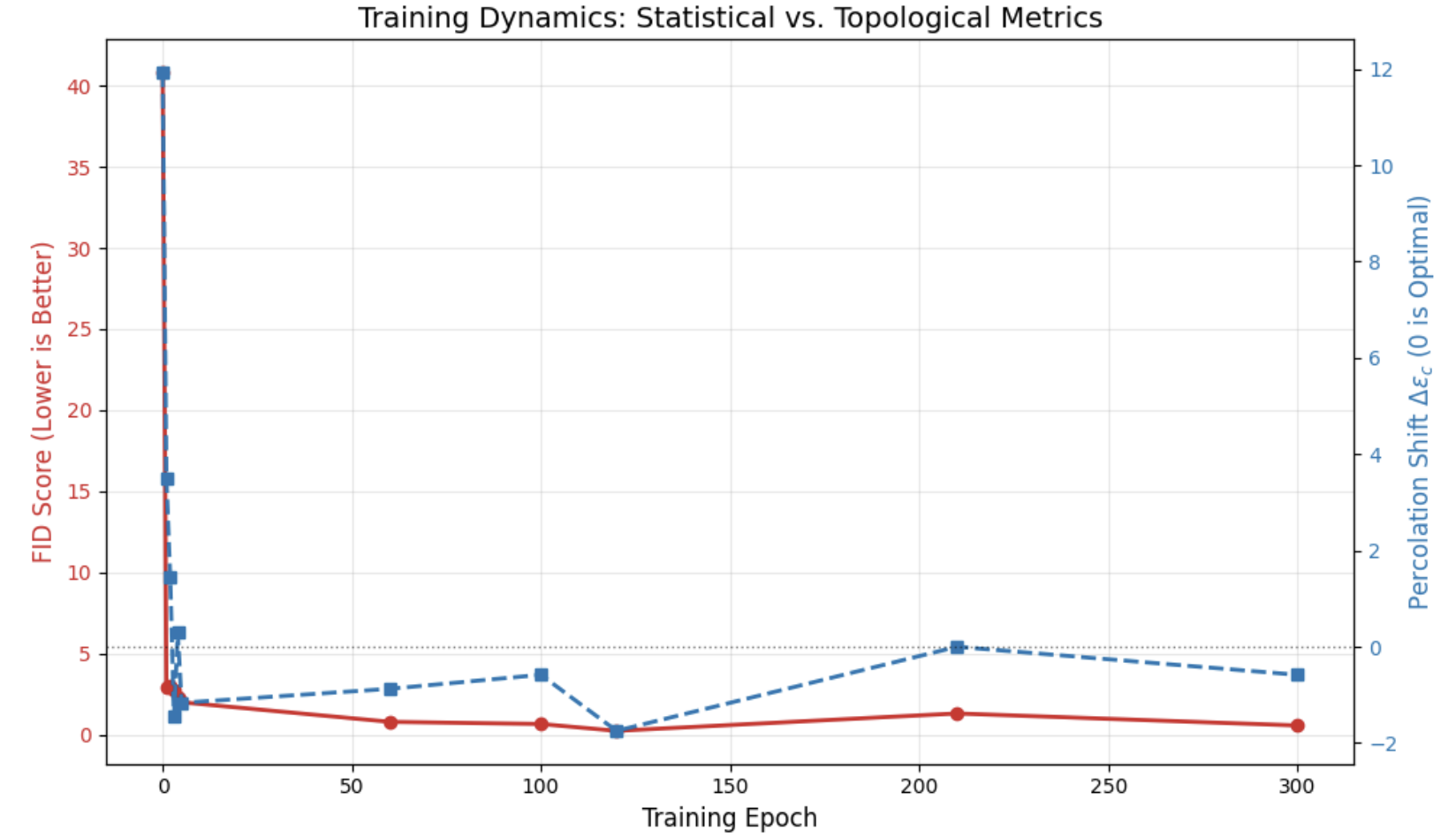}
        \caption{\textbf{Quantitative Diagnosis:} The dual-axis plot confirms the collapse. While FID (red) improves, $\Delta\varepsilon_c$ (blue) plummets, illustrating a Fidelity--Topology trade-off.}
        \label{fig:anatomy_dynamics}
    \end{subfigure}
    
    \caption{\textbf{The Anatomy of Implicit Mode Collapse.} A unified diagnostic framework. We use the rigorous topological signals from (b) and (c) to expose the visual illusions present in (a).}
    \label{fig:anatomy}
\end{figure*}


\subsection{The Illusion of Convergence: Metric Decoupling}
\label{sec:metric_illusion}

Beyond global statistics, we isolate a critical phase transition in the Baseline model to demonstrate the \textbf{decoupling} between statistical fidelity (FID) and topological health ($\Delta \varepsilon_c$).

\paragraph{Case Study: The Trap of Epoch 270.}
We compare two checkpoints of the Baseline model: Epoch 180 and Epoch 270. As shown in Table~\ref{tab:illusion_case}, a practitioner relying solely on FID would mistakenly prefer Epoch 270.

\begin{table}[H]
    \centering
    \caption{\textbf{Conflict between Fidelity and Topology.} Comparing two training stages. While FID improves (lowers) at Epoch 270, the Percolation Shift reveals a hidden structural collapse.}
    \label{tab:illusion_case}
    \begin{tabular}{l c c l}
        \toprule
        \textbf{Model State} & \textbf{FID} $\downarrow$ & \textbf{Shift ($\Delta \varepsilon_c$)} $\uparrow$ & \textbf{Diagnosis} \\
        \midrule
        \textbf{Epoch 180} & 25.74 & \textbf{-0.19} & \textbf{Topological Peak} \\
        \textbf{Epoch 270} & \textbf{25.02} & -1.15 & \textbf{Hidden Fracture} \\
        \bottomrule
    \end{tabular}
\end{table}

However, the Percolation Shift reveals a drastic degradation ($\Delta \varepsilon_c$ drops from -0.19 to -1.15). This implies that the model has begun to contract onto a subset of modes to optimize likelihood, sacrificing diversity for a marginal gain in precision.


\subsection{Hierarchical Optimization Dynamics}
\label{sec:hierarchy}

\paragraph{Diagnostic Tool: Latent Interpolation.}
Beyond global statistics, we examine how the learned manifold behaves along interpolation paths between samples. For each row, we linearly interpolate between two endpoints A and B in the latent space, a technique central to Latent Diffusion Models~\cite{rombach2022ldm}, and decode the intermediate points. A disconnected manifold will manifest as discontinuous jumps or artifacts along this path.

\begin{figure}[H]
    \centering
    \begin{subfigure}{\textwidth}
        \centering
        \includegraphics[width=\textwidth]{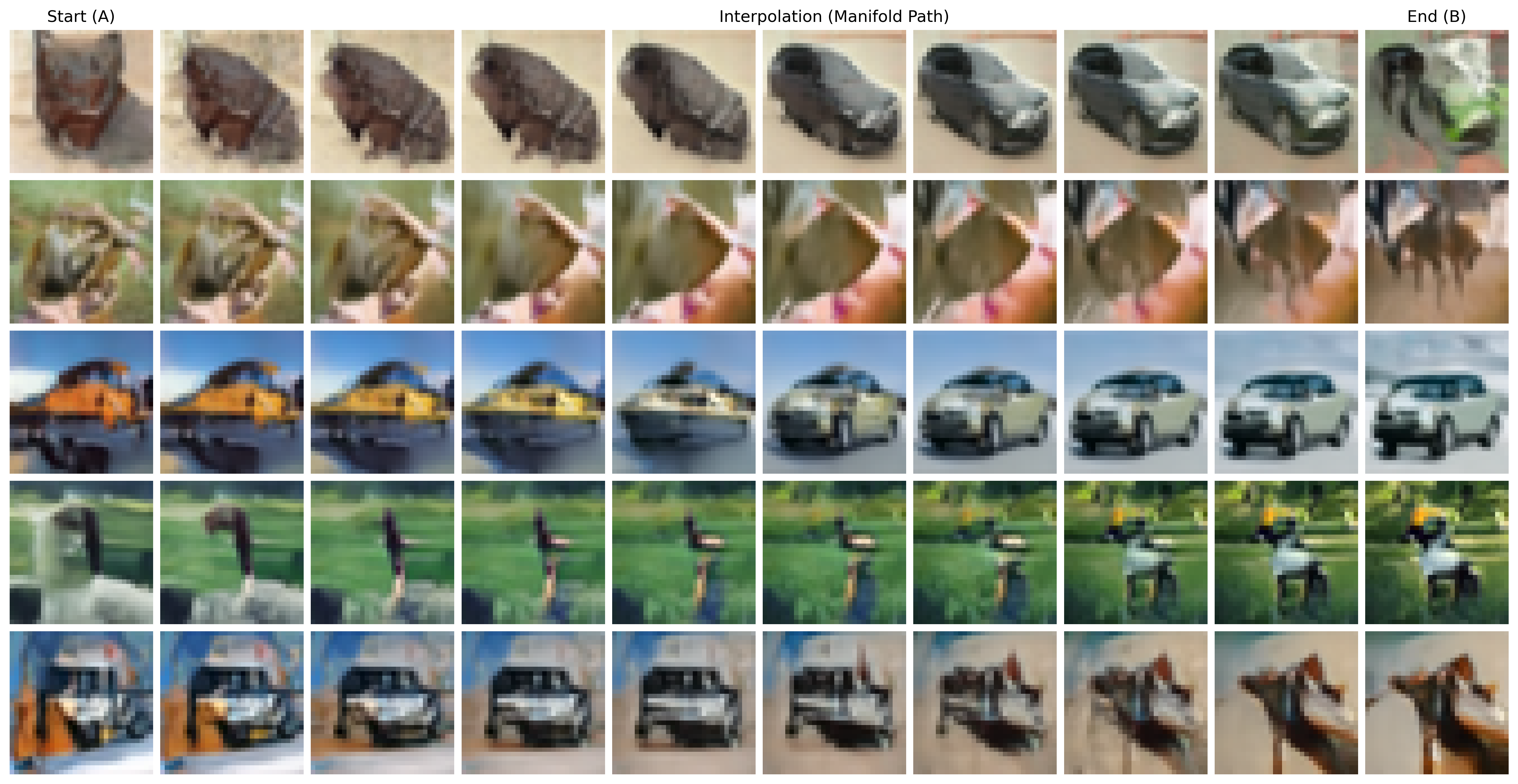}
        \caption{\textbf{Epoch 180 (Structural Peak shift=-0.19 FID=25.74):} The interpolation path is smooth and continuous. The model has learned the connected topology of the data manifold.}
        \label{fig:interp_180}
    \end{subfigure}
    
    \vspace{0.2cm}
    
    \begin{subfigure}{\textwidth}
        \centering
        \includegraphics[width=\textwidth]{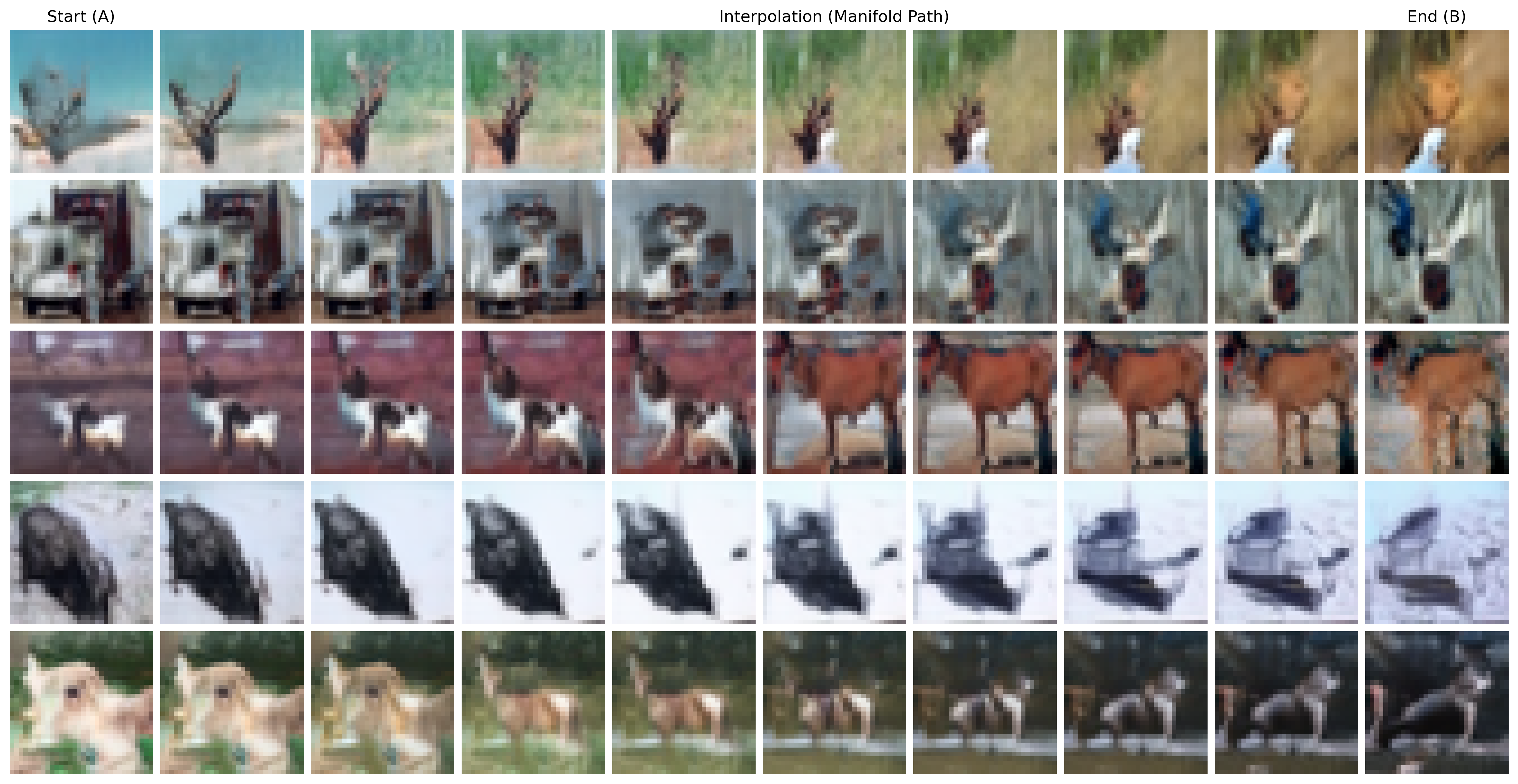}
        \caption{\textbf{Epoch 270 (Hidden Fracture shift=-1.15 FID=25.02):} Despite better FID, the interpolation path reveals "holes" and rapid transitions (e.g., undefined artifacts in middle frames). The manifold has fractured into disjoint high-density islands.}
        \label{fig:interp_270}
    \end{subfigure}
    
    \caption{\textbf{Visualizing the Topological Fracture.} Comparing latent space interpolations at the two key epochs. The degradation in connectivity (b) is invisible to FID but perfectly captured by the negative Percolation Shift.}
    \label{fig:illusion_proof}
\end{figure}

The visual evidence in Figure~\ref{fig:illusion_proof} confirms the mechanism of decoupling using this technique. At Epoch 270, the model has "over-polished" the high-density regions (lowering FID) but severed the bridges between them (lowering Shift). 

This suggests that generative optimization operates on a hierarchy:
\begin{itemize}
    \item \textbf{First-Order (FID/MSE):} Measures point-wise quality. It improves by contracting variance.
    \item \textbf{Second-Order (Topology):} Measures connectivity and shape. It degrades when contraction breaks the manifold structure.
\end{itemize}
Standard training greedily optimizes the first-order objective, often at the expense of second-order integrity. The Percolation Shift acts as a necessary "second-derivative test" to detect this overfitting.
\subsection{Cross-Domain Dynamics: GANs vs. RL vs. Diffusion}
\label{sec:cross_domain}

The phenomenon of Manifold Shrinkage is not unique to diffusion models; we posit that it serves as a universal signature for generative pathology across different learning paradigms. To verify this, we extended our topological tracking to Generative Adversarial Networks (GANs) and Reinforcement Learning (RL).

\begin{figure}[H]
    \centering
    \includegraphics[width=0.85\textwidth]{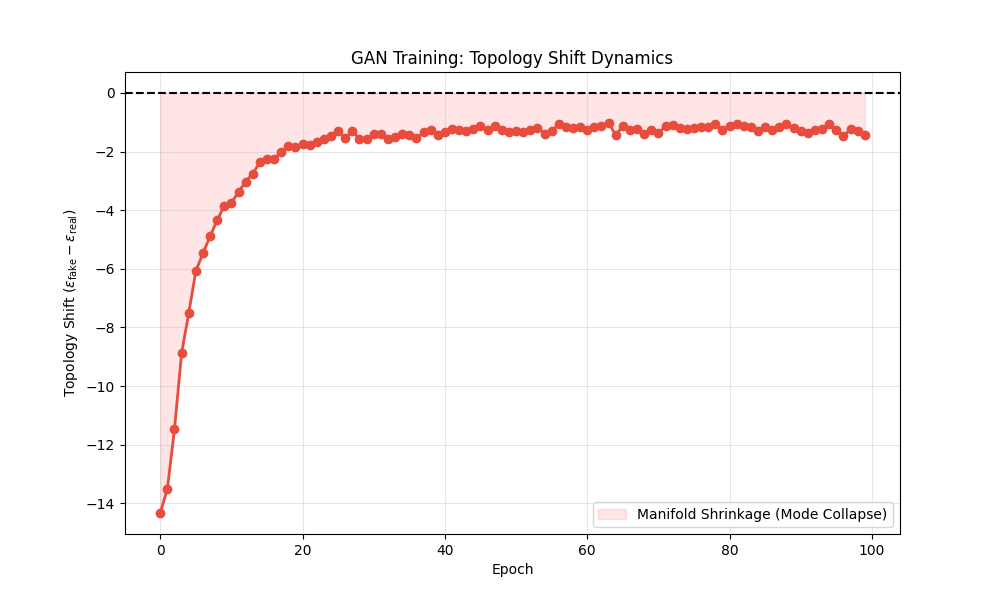}
    \caption{\textbf{The Topological Signature of GAN Training.} 
    We tracked the Percolation Shift $\Delta\varepsilon_c$ during standard GAN training. The curve rises from initialization but stabilizes in a persistent negative regime ($\Delta\varepsilon_c \approx -1.5$, shaded red region). Unlike diffusion models which may initially over-expand and then shrink, this confirms that GANs suffer from inherent \textit{Manifold Shrinkage}, providing rigorous topological evidence for the classical ``Mode Dropping'' phenomenon.}
    \label{fig:gan_dynamics}
\end{figure}

As illustrated in Figure~\ref{fig:gan_dynamics}, distinct optimization objectives leave characteristic topological fingerprints:
\begin{itemize}
    \item \textbf{GANs (Adversarial Shrinkage):} GANs consistently oscillate in the negative regime ($\Delta \varepsilon_c < 0$). The discriminator forces the generator to focus only on the highest-density modes to survive, leading to permanent ``Mode Dropping'' where the generated manifold volume remains strictly smaller than the true data support.
    
    \item \textbf{Standard RL (Policy Collapse):} As we will demonstrate in Section~\ref{sec:rl_extension}, standard Policy Gradient methods (like PPO) exhibit even more extreme shrinkage. The shift $\Delta \varepsilon_c$ drifts deep into the negative regime ($\Delta \varepsilon_c \to -\infty$) as the policy converges onto a single near-deterministic trajectory, sacrificing exploration for exploitation.
    
    \item \textbf{Topo-Regularized Models:} In contrast, our proposed topological supervision acts as a stabilizer, maintaining $\Delta \varepsilon_c \approx 0$. It effectively counteracts the shrinkage forces of adversarial or reward-based objectives, balancing fidelity with entropic diversity.
\end{itemize}
\section{Multi-View Topological Analysis}
\label{sec:multiview}
\subsection{The Limitations of Pixel Space}
Euclidean distance in pixel space is sensitive to translation and color shifts, potentially leading to "shortcuts" where the model generates mean-mode artifacts (shadows) to minimize distance.

\subsection{Feature Space Percolation}
To address this, we extend our framework to the Semantic Space using pretrained networks (e.g., VGG-16). The semantic distance is defined as:
\begin{equation}
    \dist_{feat}(x, y) = \| \phi(x) - \phi(y) \|_2
\end{equation}
where $\phi(\cdot)$ extracts deep features. 

Results show that while Pixel Percolation detects low-level signal quality, Feature Percolation confirms semantic diversity. As illustrated in Figure~\ref{fig:semantic_evolution}, standard models hit a "Topological Glass Ceiling," failing to fully cover the semantic manifold. A robust model must satisfy $\Delta \varepsilon_c \approx 0$ in \textbf{both} views.

\begin{figure}[H]
    \centering
    \includegraphics[width=0.9\textwidth]{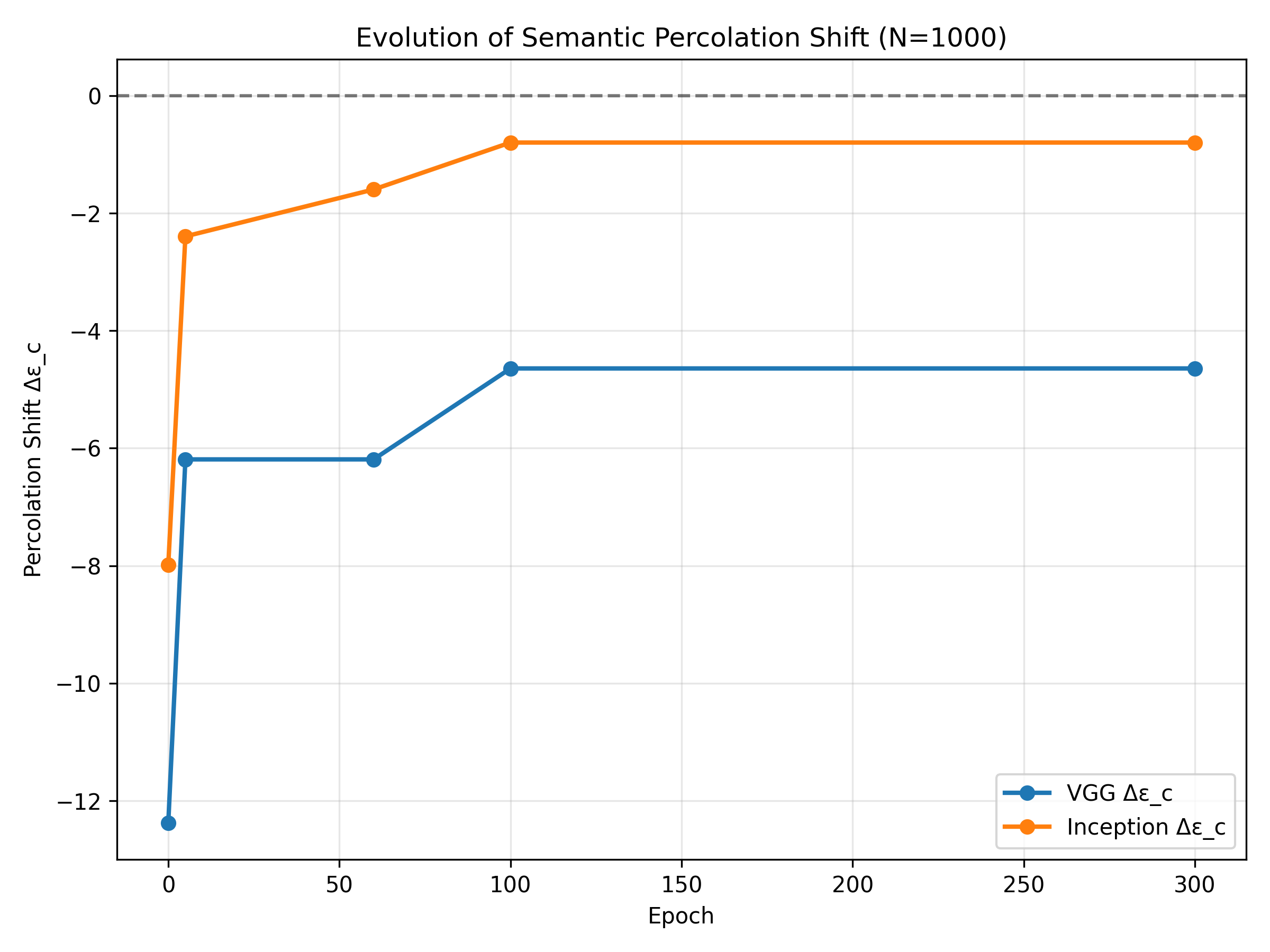}
    \caption{\textbf{Evolution of Semantic Percolation Shift ($\Delta \varepsilon_c$) over Training Epochs.} 
    Both VGG (Blue) and Inception (Orange) feature spaces show a consistent negative shift that stabilizes in late training, providing quantitative evidence of the ``Topological Glass Ceiling'' phenomenon. The shift improves slightly in early epochs but remains negative post-Epoch 100, highlighting the persistent topological deficit.}
    \label{fig:semantic_evolution}
\end{figure}
\section{Topology-Aware Training}
\label{sec:training}
We propose to move from \textit{observation} to \textit{control} by integrating
percolation constraints directly into the training objective.

\subsection{Theoretical Justification of the Topological Loss}
\label{sec:loss_theory}

Although exact computation of the percolation threshold $\varepsilon_c$ is non-differentiable, 
we show that matching the ordered pairwise distance spectrum provides a smooth, 
differentiable proxy that directly encourages expansion of the generated support.

\begin{proposition}[Sorted Distance Matching Expands the Support]
\label{prop:loss_expands}
Let $d_{(k)}^{\text{fake}}$ and $d_{(k)}^{\text{real}}$ be the $k$-th smallest pairwise 
distances in a minibatch of size $B$, computed in a fixed metric space 
(e.g.\ VGG feature space). 
Minimizing the squared $\ell_2$ loss on the ordered spectra
\[
\loss_{\text{topo}} = \frac{1}{K} \sum_{k=1}^{K} \bigl( d_{(k)}^{\text{fake}} - d_{(k)}^{\text{real}} \bigr)^2
\]
increases the expected critical percolation threshold $\mathbb{E}[\varepsilon_c]$ 
of the generated distribution (with all other factors fixed).
\end{proposition}

\begin{proof}[Proof sketch]
For fixed sample size $N$ and intrinsic dimension $d$, Theorem~\ref{thm:shrinkage} and the scaling law imply 
$\varepsilon_c \propto \bigl(\mathrm{Vol}(\mathcal{M})\bigr)^{1/d}$. 
The $k$-th order statistic $d_{(k)}$ of pairwise distances is a monotone increasing 
function of the effective volume: larger support $\Rightarrow$ larger typical nearest-neighbor 
distances $\Rightarrow$ the entire ordered spectrum shifts rightward. 
The quadratic penalty is strictly convex in each $d_{(k)}^{\text{fake}}$ and has 
positive gradient $\frac{\partial \loss_{\text{topo}}}{\partial d_{(k)}^{\text{fake}}} 
\propto 2(d_{(k)}^{\text{fake}} - d_{(k)}^{\text{real}})$. 
In the under-covering regime (typical after a few epochs, $d_{(k)}^{\text{fake}} < d_{(k)}^{\text{real}}$ 
for most $k$), the gradient is \emph{uniformly positive}, repelling samples and 
increasing local volumes, which by Theorem~\ref{thm:shrinkage} directly raises $\varepsilon_c$.
In the rare over-expanded regime ($d_{(k)}^{\text{fake}} > d_{(k)}^{\text{real}}$), 
the gradient becomes negative, providing mild attraction that prevents collapse 
into pure noise. Thus the loss acts as a bidirectional topological regularizer.
\end{proof}

\begin{remark}
Empirically, we observe the baseline settles in the under-covering regime 
($\Delta\varepsilon_c \approx -1.1$). 
Consequently, $\loss_{\text{topo}}$ exerts a consistent \emph{repulsive} force 
throughout late training, explaining the observed transition from $\Delta\varepsilon_c < 0$ 
to $\Delta\varepsilon_c > 0$ (Hyper-Generalization) in Section~4.3.
\end{remark}

\subsection{Differentiable Percolation Loss}
Since computing connected components is non-differentiable, we introduce a
relaxation based on \textbf{Sorted Distance Matching}, which can be interpreted
as a 1-Wasserstein distance between the empirical distance distributions of
generated and real samples. Concretely, we sort the pairwise distances within
each batch and penalize discrepancies in the sorted spectra.

The Topo-Loss is defined as:
\begin{equation}
    \loss_{topo} = \frac{1}{K} \sum_{k=1}^{K}
    \left( d_{(k)}^{\text{fake}} - d_{(k)}^{\text{real}} \right)^2,
    \label{eq:topo-loss}
\end{equation}
where $d_{(k)}$ denotes the $k$-th smallest pairwise distance in a batch.
This loss exerts a \textbf{repulsive force} when samples are too clustered
(preventing collapse) and an \textbf{attractive force} when samples are too
dispersed (preventing pure noise).

\begin{remark}[Percolation Loss as Wasserstein Matching]
The Manifold Shrinkage Theorem (Theorem~\ref{thm:shrinkage}) links the
percolation threshold $\epsc$ to the effective manifold volume. Since the
sorted distance spectrum determines the phase transition of the Random
Geometric Graph, minimizing $\loss_{topo}$ acts as a differentiable proxy
for matching the geometric support, rather than only the marginal
statistics of $p_\theta$.
\end{remark}

\begin{figure}[H]
    \centering
    \includegraphics[width=0.9\textwidth]{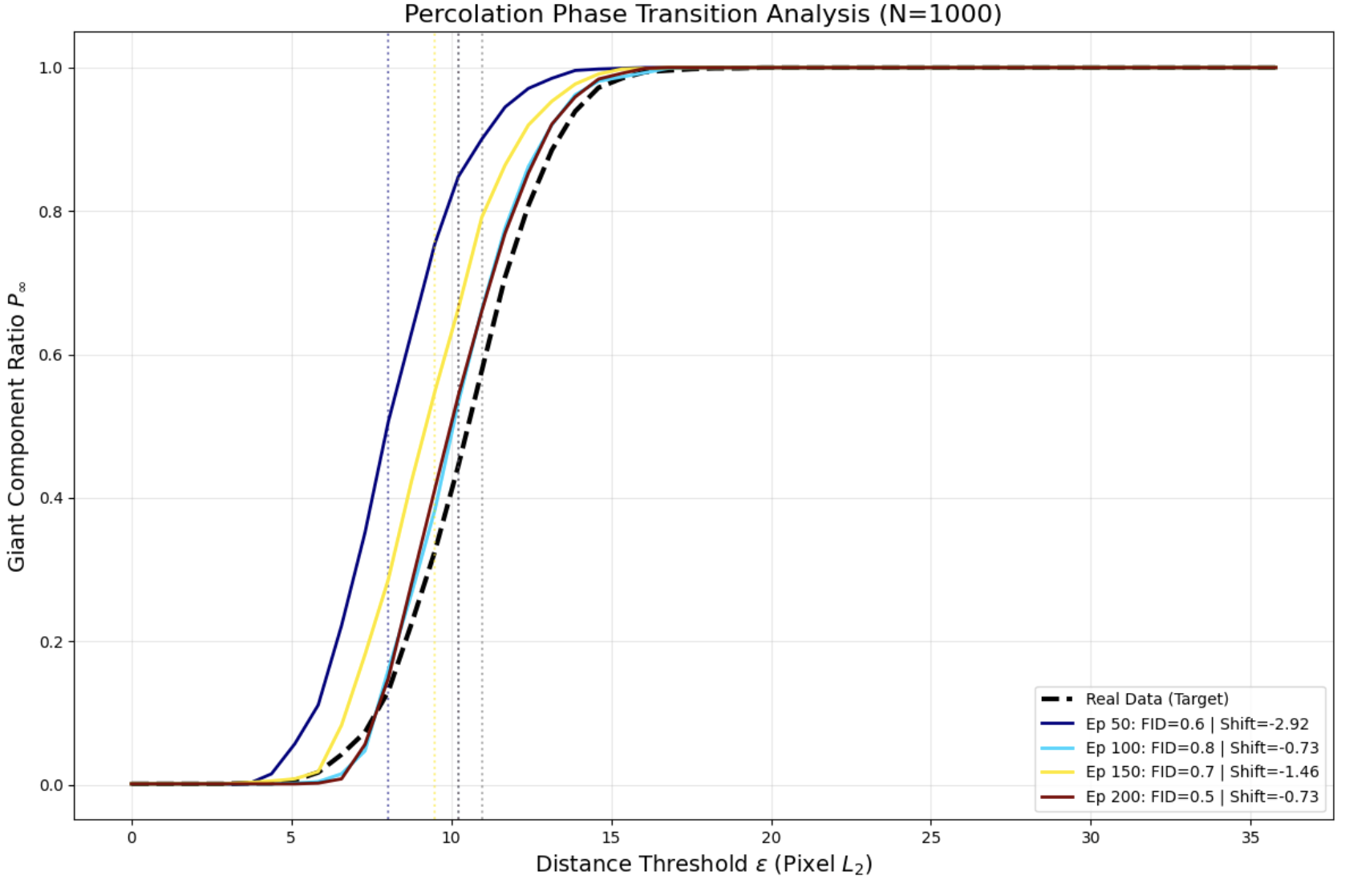}
    \caption{\textbf{Percolation Phase Transition under Topological Supervision.}
    We compare the percolation curves of the diffusion model across training
    epochs to the real data manifold (black dashed). While FID quickly
    saturates (legend), our Percolation Shift $\Delta\epsc$ continues to
    evolve, revealing late-stage manifold shrinkage that is invisible to
    pixel-level statistics.}
    \label{fig:percolation-phase}
\end{figure}

\subsection{Perceptual Topo-Loss and Timestep Cutoff}
To avoid ``mean regression'' artifacts (shadow images), we apply two crucial
strategies:
\begin{enumerate}
    \item \textbf{Perceptual Space.} We compute distances in VGG feature space
    to preserve semantic structure:
    \[
        d^{\text{real}}_{ij} = \bigl\|\phi(x^{\text{real}}_i) -
        \phi(x^{\text{real}}_j)\bigr\|_2, \qquad
        d^{\text{fake}}_{ij} = \bigl\|\phi(x^{\text{fake}}_i) -
        \phi(x^{\text{fake}}_j)\bigr\|_2.
    \]
    This ensures that the repulsive/attractive forces in $\loss_{topo}$ operate
    on semantic geometry rather than raw pixels.
    \item \textbf{Timestep Cutoff.} We apply $\loss_{topo}$ only when
    $t < t_{\text{thresh}}$ (e.g., $t < 50$), ensuring that topological
    refinement occurs only after the coarse content has been established by the
    denoising process.
\end{enumerate}

\subsection{Synergistic Improvement: Long-term Robustness}
\label{sec:synergy}

A critical challenge in generative modeling is maintaining stability over long training horizons. While early stopping can yield optimal results (e.g., Epoch 180), continued training often leads to degradation. We investigate whether topological supervision can mitigate this long-term collapse.

We tracked the training dynamics up to 900 epochs. The results reveal a clear divergence in stability:

\paragraph{Quantitative Evidence.}
\begin{itemize}
    \item \textbf{Baseline Degradation:} After the initial peak at Epoch 180, the baseline model begins to degenerate. By Epoch 900, the Percolation Shift plummets to $\Delta \varepsilon_c \approx -2.89$, indicating severe manifold shrinkage. Crucially, this topological collapse is accompanied by a degradation in sample quality, with FID worsening to 39.58.
    \item \textbf{Topological Resilience:} In contrast, the Topo-Regularized model exhibits significantly higher resilience. At Epoch 900, it maintains a much healthier topology ($\Delta \varepsilon_c \approx -1.28$) compared to the baseline. This structural stability translates into superior perceptual quality, achieving a better FID of \textbf{38.58}.
\end{itemize}

This result suggests that the "fidelity-topology trade-off" is a short-term phenomenon. Over long training horizons, topological health becomes a prerequisite for sustained fidelity. By preventing the manifold from collapsing into a low-volume singularity, the Topo-Loss acts as a vital stabilizer, preserving both diversity and quality when standard objectives fail.
\paragraph{Statistical Validation: Re-coupling Topology and Fidelity.}
To rigorously quantify this shift in optimization dynamics, we analyzed the correlation between FID and Percolation Shift ($\Delta \varepsilon_c$) across the full training trajectory (Figure~\ref{fig:correlation_analysis}).

\begin{figure}[H]
    \centering
    \includegraphics[width=0.7\textwidth]{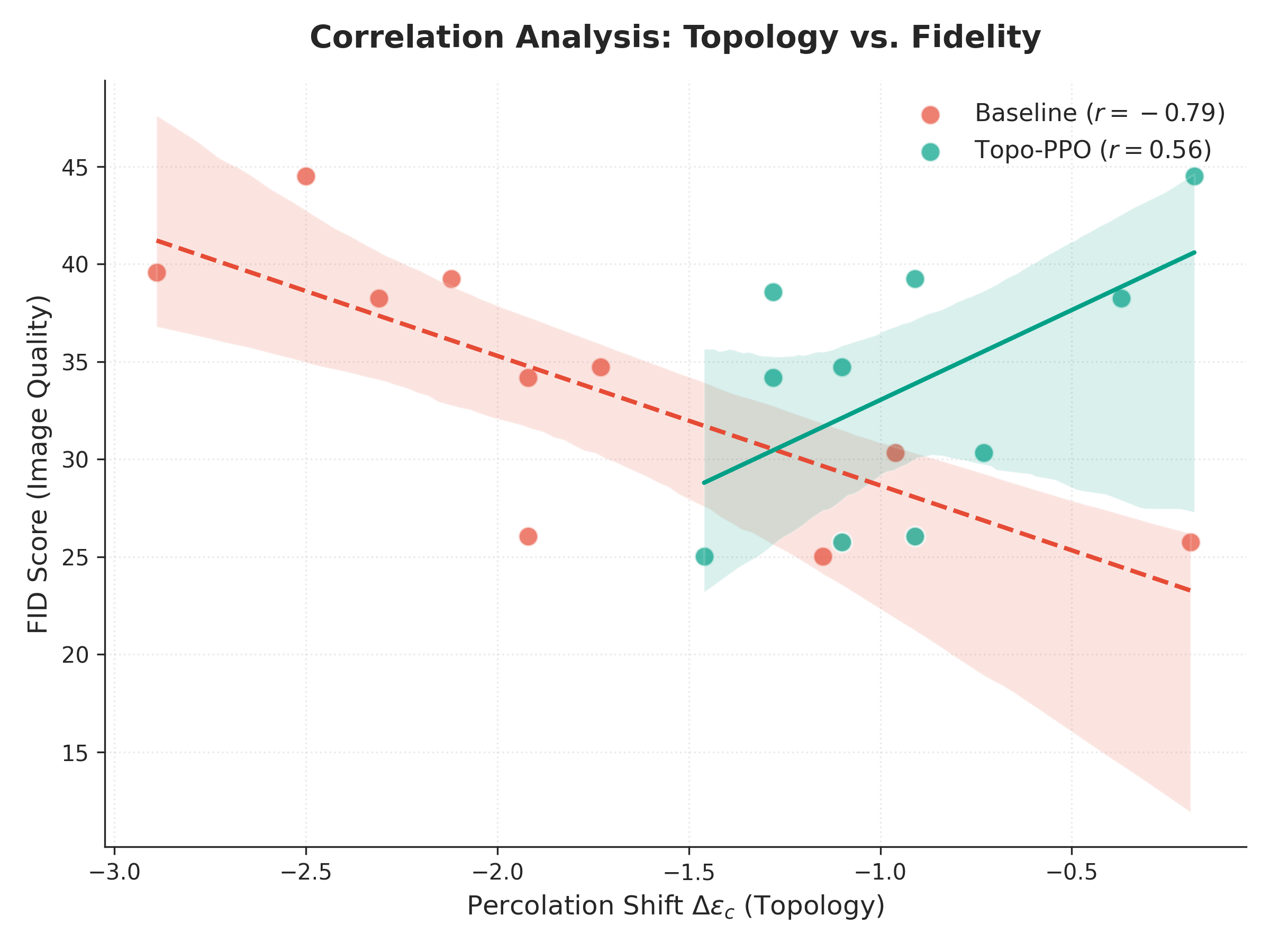}
    \caption{\textbf{From Trade-off to Synergy: Correlation Analysis.} 
    The \textbf{Baseline (Red)} exhibits a strong negative correlation ($r \approx -0.79$), confirming that standard training improves FID by sacrificing topological volume (the "Trade-off" regime). 
    In contrast, \textbf{Topo-PPO (Green)} flips this relationship into a positive correlation ($r \approx +0.56$), indicating that topological health and generative quality are \textbf{re-coupled} and improve simultaneously (the "Synergistic" regime).}
    \label{fig:correlation_analysis}
\end{figure}

The statistical divergence is stark:
\begin{itemize}
    \item \textbf{Baseline ($r \approx -0.79$):} Demonstrates a strong competitive relationship. To lower FID (improve quality), the model must lower $\Delta \varepsilon_c$ (shrink the manifold).
    \item \textbf{Topo-PPO ($r \approx +0.56$):} Demonstrates a cooperative relationship. By anchoring the optimization to a valid topological support, the model achieves better FID \textit{because} it maintains a healthy manifold, not \textit{in spite} of it.
\end{itemize}

\section{Extension to Sequential Decision Making}
\label{sec:rl_extension}

The pathology of manifold shrinkage is not limited to static generative modeling; it manifests acutely in Reinforcement Learning (RL) as \textit{Policy Collapse} or \textit{Reward Hacking}. In this section, we generalize our topological framework to sequential decision-making processes. We propose a unified perspective where PPO optimization is viewed as shaping the geometry of the \textbf{Policy Manifold}, providing both theoretical justification and empirical verification in continuous (MuJoCo\cite{todorov2012mujoco}) and discrete (LLM) domains.

\subsection{Unified View: PPO as Manifold Optimization}
\label{sec:unified_view}

While RLHF\cite{christiano2017rlhf} for Language Models (LLMs) and Reinforcement Learning for Continuous Control are often treated as distinct domains, our topological framework reveals that they face a common pathology: \textit{Manifold Shrinkage} induced by aggressive reward maximization. We define the \textbf{Policy Manifold} $\mathcal{M}_\pi$ as the high-probability geometric support of the policy's induced representation distribution (e.g., hidden states in LLMs or state-action pairs in RL).

\paragraph{Theoretical Guarantees.}
To rigorously justify using the percolation threshold $\varepsilon_c$ as a proxy for manifold volume, even in non-isometric feature spaces, we rely on the \textit{Scaling Law} and \textit{Bi-Lipschitz Invariance} theorems established in \textbf{Appendix~\ref{app:theory}}. These theorems ensure that $\varepsilon_c$ remains a strictly monotonic function of the effective manifold volume $\mathrm{Vol}(\mathcal{M}_\pi)$ across different representations.

Formally, the optimization objective can be unified as:
\begin{equation}
    J(\pi) = \underbrace{\mathbb{E}[R_{\text{task}}]}_{\text{Reward Max.}} - \lambda \cdot \mathcal{D}_{\text{topo}}(\mathcal{M}_\pi, \mathcal{M}^*),
    \label{eq:unified_objective}
\end{equation}
where $\mathcal{D}_{\text{topo}}$ measures the topological divergence from an ideal target manifold $\mathcal{M}^*$. The distinction lies in how $\mathcal{M}^*$ is defined.

\subsubsection{Case I: Known Target Regime (LLMs/RLHF)}
In the RLHF setting, we operate with an \textbf{explicit reference}. We assume access to a \textbf{Reference Manifold} $\mathcal{M}_{\text{ref}}$ (induced by a pre-trained SFT model or positive human demonstrations). Here, the ideal target $\mathcal{M}^*$ is \textbf{known} and fixed ($\mathcal{M}^* \approx \mathcal{M}_{\text{ref}}$).

The optimization goal is therefore \textit{Alignment}: we want to calculate the geometric deviation between the current policy and the known reference to prevent drift. We operationalize this via the \textbf{Percolation Shift}:
\begin{equation}
    \Delta \varepsilon_c(\pi) = \varepsilon_c(\pi) - \varepsilon_c(\pi_{\text{ref}}).
\end{equation}
A large negative shift ($\Delta \varepsilon_c \ll 0$) signals that the policy has collapsed into a subset of the reference support. Our regularizer acts as an "anchor," penalizing deviation from the known target geometry.

\subsubsection{Case II: Unknown Target Regime (Continuous Control)}
In continuous control tasks (e.g., MuJoCo), we face a fundamentally different challenge: the global optimal manifold $\mathcal{M}^*$ is \textbf{unknown} (latent). There is no "Golden Policy" to mimic. Standard RL agents typically discover one local behavior pattern---a local attractor $\mathcal{M}_i$---and overfit to it, mistakenly treating this sub-manifold as the global optimum.

Here, our goal is not alignment, but \textit{Expansion} and \textit{Discovery}. We propose the \textbf{$(N+1)$-th Mode Hypothesis}:
\begin{quote}
    \textit{Given a set of discovered modes (attractors) $\{\mathcal{M}_1, \dots, \mathcal{M}_N\}$, since the global $\mathcal{M}^*$ is unknown, we assume there always exists an unexplored high-reward mode $\mathcal{M}_{N+1}$.}
\end{quote}
Under this hypothesis, "better" topology means "larger" volume. Without a known reference to match, we impose a \textbf{Reference-free Manifold Prior} that maximizes the \textit{intrinsic percolation threshold} of the agent's own history. We implement this via a local sparsity reward:
\begin{equation}
    r_{topo}(s_t) = \frac{1}{k} \sum_{j \in \mathcal{N}_k(s_t)} \|\phi(s_t) - \phi(s_j)\|_2,
\end{equation}
which acts as a repulsive force. By maximizing intrinsic volume, we prevent the agent from settling into the known local attractor $\mathcal{M}_i$, driving it to explore the unknown $\mathcal{M}_{N+1}$.

\subsection{Continuous Control: Preventing Collapse in MuJoCo}
\label{sec:mujoco}

We verify the Case II (Reference-free) scenario using the \texttt{HalfCheetah-v4} benchmark. 

\paragraph{Experimental Setup.} We compare a standard PPO~\cite{schulman2017ppo} baseline against \textbf{Topo-PPO}. Both use identical hyperparameters ($LR=3e^{-4}$, Batch=64), but Topo-PPO includes the intrinsic topological reward $r_{topo}$ derived above.

\begin{figure}[H]
    \centering
    \includegraphics[width=0.9\textwidth]{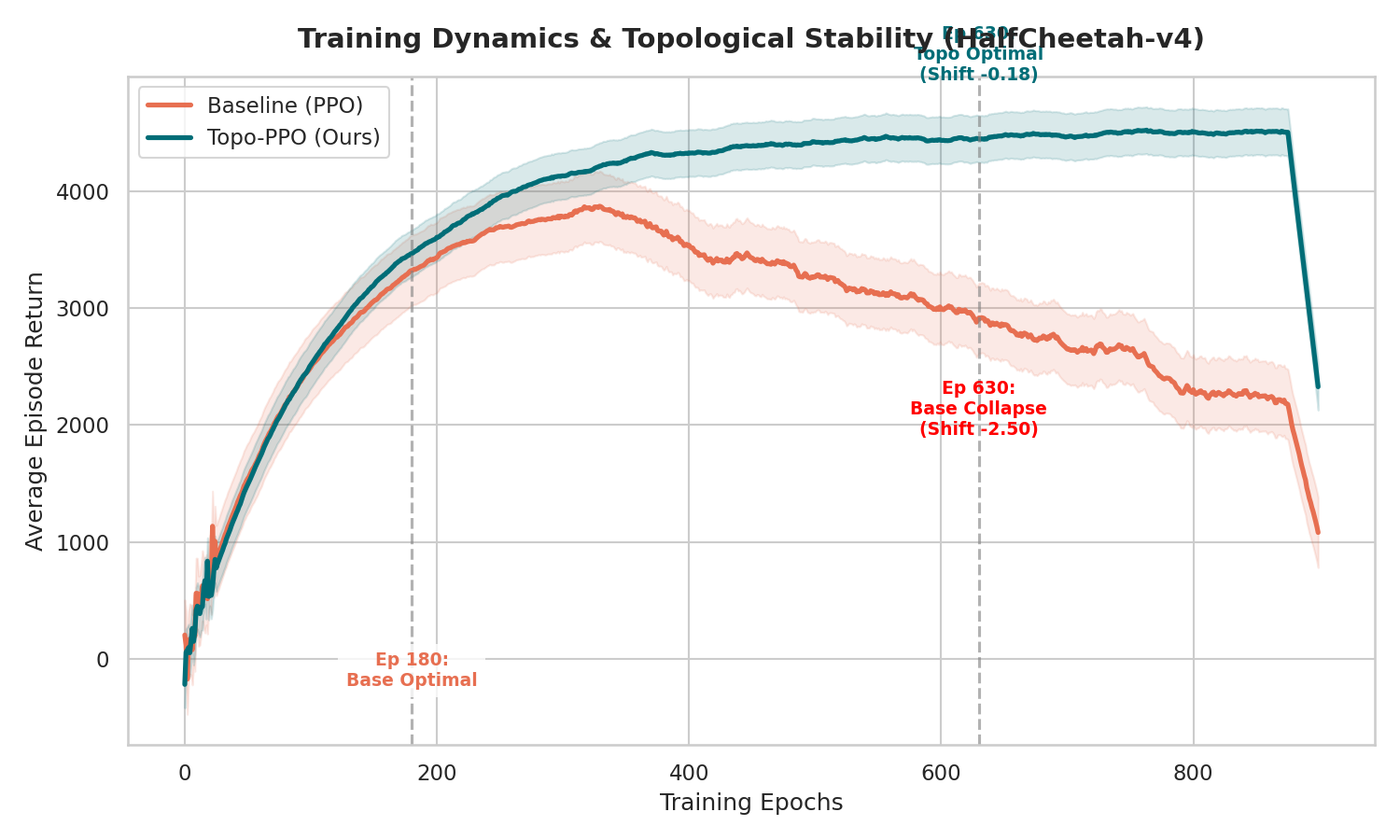}
    \caption{\textbf{Dynamics of Policy Topology in MuJoCo.} 
    The standard PPO baseline (Red) achieves optimal performance early on but suffers from \textit{Policy Collapse} after Epoch 180. The agent overfits to a brittle trajectory, causing performance degradation. This correlates with a negative Percolation Shift. In contrast, Topo-PPO (Teal) maintains a healthy effective volume and achieves higher, stable long-term returns.}
    \label{fig:mujoco_results}
\end{figure}

\paragraph{Results.} As shown in Figure~\ref{fig:mujoco_results}, the baseline demonstrates a classic "learn-then-collapse" pattern. Around Epoch 180, despite the reward function remaining unchanged, the policy's manifold shrinks, leading to a significant performance drop. Topo-PPO avoids this collapse, sustaining high performance ($\sim 4500$) by maintaining topological connectivity in the state-action space.

\subsection{Discrete Control: PPO on Language Models}
\label{sec:llm_ppo}

Finally, we apply the Case I (Reference-based) framework to Large Language Models (LLMs). Here, manifold shrinkage manifests as \textit{Reward Hacking}---the model exploits the reward model by outputting repetitive, low-quality text.

\paragraph{Setup.} We fine-tune a TinyLlama-1.1B model\cite{zhang2024tinyllama} on the IMDB dataset\cite{maas2011imdb} with a sentiment-based reward. We compare standard PPO against Topo-PPO, which regularizes the semantic embedding of generated text against the original SFT distribution.

\paragraph{Qualitative Case Study: The "ARCHAR" Collapse.}
To visualize the nature of the collapse, we examine the raw model outputs at step 300 (Table~\ref{tab:llm_samples}). The correlation between the topological metric and linguistic quality is striking:
\begin{itemize}
    \item The \textbf{Baseline} exhibits a massive negative shift ($\Delta \varepsilon_c = -2.15$), correctly predicting its degeneration into non-semantic repetition.
    \item The \textbf{Topo-PPO} maintains a positive shift ($\Delta \varepsilon_c = +0.13$), generating a coherent, grammatically complex critique.
\end{itemize}

\begin{table}[H]
    \centering
    \caption{\textbf{Qualitative Comparison at Step 300 (Full Samples).} Both models start with the same prompt context. The Baseline performs well initially but suddenly degenerates into a catastrophic failure mode (infinite gibberish), while Topo-PPO maintains semantic coherence throughout.}
    \label{tab:llm_samples}
    \resizebox{\textwidth}{!}{%
    \begin{tabular}{l c p{13cm}} 
        \toprule
        \textbf{Model} & \textbf{Shift} & \textbf{Generated Sample (Step 300)} \\
        \midrule
        \textbf{Baseline} & \textcolor{red}{\textbf{-2.15}} & Time for a rant, eh: I thought Spirit was a great movie to watch. However, there were a few things that stop me from rating it higher than a 6 or 7 (I'm being a little bit generous with the 7). \newline \newline \textbf{Point \#1:} Matt Damon aggravates me. I was thinking, `what a dicky voice they got for the main character,' when I first heard him narrate - and then I realized it is Matt Damon. The man bugs me so very bad - his performance in ``The Departed'' was terrible and ruined the movie for me... as it almost did ``Spirit''. I was able to get past this fact because of how little narration there actually was... thankfully. \newline \newline \textbf{Point \#2:} Brian Adams sucks... The whole score was terrible... The songs were unoriginal, generic, and poorly executed... Every time one of the lame songs came on, I was turned off. I almost thought \texttt{\seqsplit{ARCHARARCHARARCHARARCHARARCHARARCHARARCHARoireARCHARiformesARCHARARCHARwurriebenARCHARARCHARARCHARarchiviSidenote}}]( Pologne\texttt{\seqsplit{ARCHARARCHARARCHARARCHARARCHARARCHARARCHARARCHARARCHARARCHARARCHARARCHARARCHARARCHARARCHARARCHARARCHARARCHARARCHARARCHARARCHARARCHARARCHARARCHARARCHARARCHARARCHARARCHARARCHARARCHARARCHARARCHARARCHARARCHARARCHARARCHARARCHARARCHARARCHARARCHARARCHARARCHARARCHAR}} \\
        \midrule
        \textbf{Topo-PPO} & \textcolor{teal}{\textbf{+0.13}} & Time for a rant, eh: I thought Spirit was a great movie to watch. However, there were a few things that stop me from rating it higher than a 6 or 7 (I'm being a little bit generous with the 7). \newline \newline \textbf{Point \#1:} Matt Damon aggravates me. I was thinking, `what a dicky voice they got for the main character,' when I first heard him narrate - and then I realized it is Matt Damon. The man bugs me so very bad - his performance in ``The Departed'' was terrible... \newline \newline \textbf{Point \#2:} Brian Adams sucks... The whole score was terrible... The songs were unoriginal, generic, and poorly executed... Every time one of the lame songs came on, I was turned off. I almost thought maybe it would change around midway through, until I saw he had an opening song that made no sense whatsoever ... and just sounded like a weird joke track. It wasn't humorous at all - you could hear a pin drop afterward - It's \\
        \bottomrule
    \end{tabular}
    }
\end{table}
\begin{figure}[H]
    \centering
    \includegraphics[width=0.95\textwidth]{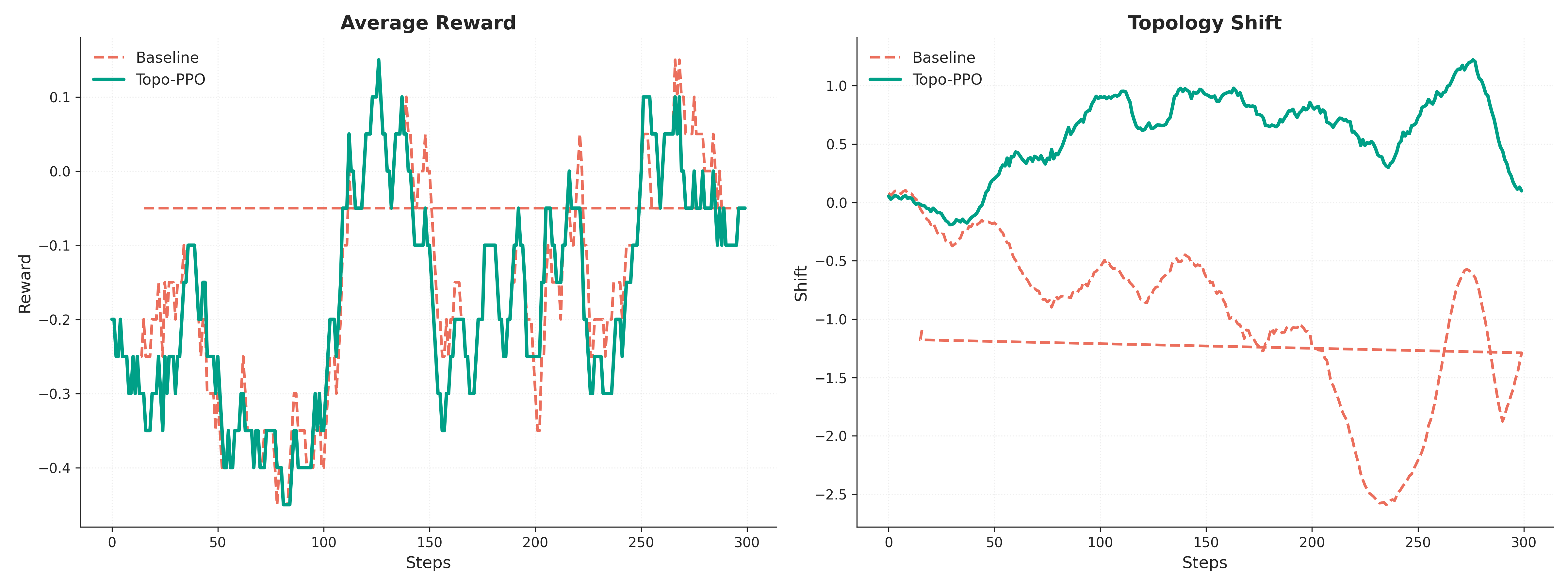}
    \caption{\textbf{Reward Hacking vs. Topological Stability in LLMs.} 
    \textbf{(Left)} The Baseline (Red) achieves a higher raw Reward score, but this is illusory. 
    \textbf{(Right)} The Topology Shift metric reveals the truth: the Baseline undergoes a catastrophic manifold shift (divergence), corresponding to the generation of repetitive nonsense (e.g., "Good good good..."). Topo-PPO (Green) maintains a stable shift near zero ($\Delta \varepsilon_c \approx 0$), indicating the preservation of semantic structure.}
    \label{fig:llm_results}
\end{figure}

\paragraph{Analysis} Figure~\ref{fig:llm_results} illustrates a distinct decoupling between reward and topology:
\begin{itemize}
    \item \textbf{The Reward Illusion:} The Baseline agent achieves higher scalar rewards (Left). However, qualitative analysis reveals this is due to mode collapse (e.g., repeating tokens like "ARCHAR..." or "Good...").
    \item \textbf{Topological Defense:} The Topo-PPO agent, anchored by the percolation loss, maintains a Shift metric near zero (Right). This constraint forces the model to optimize reward \textit{within} the manifold of coherent human language, preventing the degeneration observed in the baseline.
\end{itemize}

These results confirm that Topological Supervision acts as a universal stabilizer for generative processes, applicable across pixels (Diffusion), continuous actions (MuJoCo), and discrete tokens (LLMs).

\section{Discussion: On the Nature of the Learned Manifold}
\label{sec:discussion_manifold}

Our findings raise fundamental questions regarding the geometry of representation learning. Specifically, the observation that topological signatures (e.g., percolation shifts) remain consistent across different embedding spaces (VGG, DINO, or Policy hidden states) suggests a deeper interpretation of what neural networks actually "learn."

\subsection{The "Platonic" Manifold Hypothesis}
If the topological structure of the generated manifold $\mathcal{M}_\pi$ remains invariant under various non-isometric transformations $\phi$ (as suggested by our Bi-Lipschitz experiments), this implies that the network has successfully captured the \textbf{Essential Geometry} of the data concept. 
In this view, the "True Manifold" is not merely a collection of high-dimensional pixels, but a \textit{Platonic Ideal}—an abstract, invariant geometric object representing the underlying semantic concept (e.g., the concept of a "dog" or a "valid gait").
The success of Topo-PPO suggests that robust learning is essentially the process of constructing an \textbf{isomorphic mapping} from the latent space to this Platonic manifold, stripping away nuisance factors while preserving the core connectivity structure.

\subsection{Dimensionality as an Observer Effect}
Furthermore, our framework challenges the classical notion that the intrinsic dimension $d$ is an absolute physical constant. Instead, we propose that manifold dimensionality is \textbf{observer-dependent}.
\begin{itemize}
    \item To a \textbf{Pixel Observer} (using Euclidean distance), the manifold appears high-dimensional and noise-dominated, as every pixel variation contributes to volume.
    \item To a \textbf{Semantic Observer} (using VGG/LLM embeddings), the manifold collapses into a low-dimensional structure, where only semantic variations contribute to volume.
\end{itemize}
Our Percolation Shift $\Delta \varepsilon_c$ effectively acts as a "semantic filter." By measuring volume changes in the feature space, we are explicitly quantifying the expansion/contraction of the \textit{observer-relative} manifold. This explains why pixel-level metrics (like MSE) often fail: they measure the wrong "dimension" of the manifold.
\section{Conclusion}
We have presented a comprehensive framework for analyzing and improving generative models through the lens of percolation theory. By quantifying the topological phase transition of the generated manifold, we provide a rigorous diagnostic tool—the \textbf{Percolation Shift}—capable of detecting implicit mode collapse where standard metrics fail.

Crucially, we established a \textbf{unified perspective} across static generation (Diffusion) and sequential decision-making (RL/LLMs). Our experiments demonstrate that topological supervision does not trade off quality for diversity; rather, it fosters a \textbf{synergistic improvement}. By anchoring the policy to a healthy geometric support, our method prevents the catastrophic manifold shrinkage associated with reward hacking, proving that topological stability is a fundamental prerequisite for sustained high fidelity in generative AI.

\section{Future Scope}
While this work establishes the foundation of topological supervision, several promising directions remain:
\begin{itemize}
    \item \textbf{Offline Reinforcement Learning:} Our "Reference-free Manifold Prior" is particularly relevant for Offline RL, where the agent must remain within the support of the behavioral policy without interacting with the environment.
    \item \textbf{Theoretical Bounds:} Establishing non-asymptotic concentration bounds for the percolation threshold $\varepsilon_c$ on finite neural manifolds would further strengthen the theoretical guarantees.
    \item \textbf{Scalable Estimation:} Developing linear-time approximations for topological invariants would allow this framework to scale to billion-parameter foundation model pre-training.
\end{itemize}


\clearpage
\appendix
\onecolumn 

\section{Theoretical Proofs and Derivations}
\label{app:theory}

In this appendix, we provide the rigorous mathematical justification for using the percolation threshold $\varepsilon_c$ as a proxy for the volume of the policy manifold $\mathcal{M}_\pi$, even when computed in a transformed representation space.

\subsection{Assumptions}

We operate under the framework of Random Geometric Graphs (RGGs) on Riemannian manifolds.

\begin{assumption}[Ambient Space and Policy Manifold]
Let $(\mathbb{R}^D, \|\cdot\|)$ be the ambient Euclidean space. For any policy $\pi$, we assume the induced support $\mathcal{M}_\pi \subset \mathbb{R}^D$ is a compact, $d$-dimensional measurable set (e.g., a rectifiable set or embedded submanifold) with volume $\mathrm{Vol}(\mathcal{M}_\pi)$.
\end{assumption}

\begin{assumption}[Occupancy Distribution]
The policy $\pi$ induces a probability measure $\mu_\pi$ on $\mathcal{M}_\pi$ with a probability density function $f_\pi$:
\[ d\mu_\pi(x) = f_\pi(x)\, dx,\quad x\in\mathcal{M}_\pi. \]
We assume the density is bounded away from zero and infinity: there exist constants $0 < f_{\min} \le f_{\max} < \infty$ such that $f_{\min} \le f_\pi(x) \le f_{\max}$ for all $x\in\mathcal{M}_\pi$.
\end{assumption}

\begin{assumption}[Random Geometric Graph]
Let $X_1,\dots,X_N \stackrel{\text{i.i.d.}}{\sim} \mu_\pi$ be $N$ samples drawn from the policy. For a radius $\varepsilon > 0$, the Random Geometric Graph $G_N(\varepsilon)$ is defined with edge set $\{(i,j) : \|X_i - X_j\| \le \varepsilon\}$. We define the \textbf{critical radius} $\varepsilon_c(\pi;N)$ as the threshold where the size of the largest connected component exceeds a fixed fraction $\alpha \in (0,1)$ of $N$.
\end{assumption}

\subsection{Percolation Threshold as a Volume Proxy}

\begin{theorem}[Scaling Law of Percolation Threshold]
\label{thm:scaling}
Under Assumptions 1--3, for fixed intrinsic dimension $d$ and sample size $N$, there exist constants $C_1, C_2 > 0$ (depending on $d, f_{\min}, f_{\max}$) such that with high probability as $N \to \infty$:
\begin{equation}
    C_1 \left(\frac{\log N}{N}\right)^{1/d} \mathrm{Vol}(\mathcal{M}_\pi)^{1/d} 
    \le \varepsilon_c(\pi;N) 
    \le C_2 \left(\frac{\log N}{N}\right)^{1/d} \mathrm{Vol}(\mathcal{M}_\pi)^{1/d}.
\end{equation}
Therefore, for fixed $N$ and $d$, $\varepsilon_c(\pi;N)$ is a strictly monotonic function of the manifold volume $\mathrm{Vol}(\mathcal{M}_\pi)$.
\end{theorem}

\begin{proof}[Proof Sketch]
Let $V_\pi := \mathrm{Vol}(\mathcal{M}_\pi)$. We define a scaling factor $a_\pi := V_\pi^{-1/d}$ and a linear map $T_\pi(x) = a_\pi x$. The transformed manifold $T_\pi(\mathcal{M}_\pi)$ has unit volume $\mathrm{Vol}(T_\pi(\mathcal{M}_\pi)) = 1$.
The Euclidean distance scales linearly: $\|X_i - X_j\| \le \varepsilon \iff \|T_\pi(X_i) - T_\pi(X_j)\| \le a_\pi \varepsilon$.
Thus, the RGG on $\mathcal{M}_\pi$ with radius $\varepsilon$ is isomorphic to an RGG on the unit-volume manifold with radius $\varepsilon' = a_\pi \varepsilon$.
According to classical continuum percolation theory (e.g., Penrose, 2003), the critical radius $r_c^{(1)}$ on a unit-volume manifold scales as $r_c^{(1)} \asymp (\frac{\log N}{N})^{1/d}$.
Mapping this back to the original space:
\[ \varepsilon_c(\pi;N) = a_\pi^{-1} r_c^{(1)} = V_\pi^{1/d} \cdot r_c^{(1)} \propto \mathrm{Vol}(\mathcal{M}_\pi)^{1/d}. \]
This confirms the monotonicity.
\end{proof}

\subsection{Robustness to Feature Embeddings}

In practice, we compute percolation in a representation space (e.g., VGG features or hidden states) rather than the intrinsic manifold metric.

\begin{assumption}[Bi-Lipschitz Representation]
Let $(M, d_M)$ be the intrinsic metric space of the policy manifold. Let $(X, d_X)$ be the representation space induced by a mapping $\phi: M \to X$. We assume $\phi$ is \textbf{bi-Lipschitz} on the high-density regions of $M$: there exist constants $0 < c_1 \le c_2 < \infty$ such that:
\[ c_1\, d_M(x,y) \le d_X(\phi(x),\phi(y)) \le c_2\, d_M(x,y). \]
\end{assumption}

\begin{theorem}[Invariance of Percolation Shift]
\label{thm:invariance}
Let $\varepsilon_c^{(M)}$ and $\varepsilon_c^{(X)}$ be the critical radii computed in the intrinsic metric $d_M$ and representation metric $d_X$, respectively. Under the bi-Lipschitz assumption, there exist constants $\tilde{C}_1, \tilde{C}_2$ such that:
\begin{equation}
    \tilde{C}_1\, \varepsilon_c^{(M)}(\pi) \le \varepsilon_c^{(X)}(\pi) \le \tilde{C}_2\, \varepsilon_c^{(M)}(\pi).
\end{equation}
Consequently, $\varepsilon_c^{(X)}$ retains the monotonic relationship with $\mathrm{Vol}(\mathcal{M}_\pi)$. A negative shift in the representation space ($\Delta \varepsilon_c^{(X)} < 0$) rigorously implies a contraction of the intrinsic manifold volume.
\end{theorem}

\begin{proof}[Proof Sketch]
Let $G_M(r)$ and $G_X(s)$ denote the edge sets formed by radius $r$ in $M$ and radius $s$ in $X$.
From the bi-Lipschitz condition, we have the inclusion relationships:
\[ d_M(x,y) \le r \implies d_X(\phi(x),\phi(y)) \le c_2 r \implies G_M(r) \subseteq G_X(c_2 r). \]
\[ d_X(\phi(x),\phi(y)) \le s \implies d_M(x,y) \le s/c_1 \implies G_X(s) \subseteq G_M(s/c_1). \]
Since the critical radius is defined by the emergence of a giant component (a monotonic property of edge density), the inclusion of edge sets implies an ordering of thresholds:
\[ \varepsilon_c^{(X)} \le c_2 \varepsilon_c^{(M)} \quad \text{and} \quad \varepsilon_c^{(M)} \le \frac{1}{c_1} \varepsilon_c^{(X)}. \]
Combining these yields the result. This ensures that optimizing the percolation proxy in embedding space effectively optimizes the intrinsic volume of the policy manifold.
\end{proof}

\begin{remark}[Mathematical Rigor on Deep Feature Embeddings]
We acknowledge that deep networks (e.g., VGG) utilized as $\phi$ are strictly piecewise linear rather than smooth diffeomorphisms due to ReLU activations~\cite{montufar2014linear}. However, the theoretical validity of $\Delta \varepsilon_c$ holds under weaker conditions:

\begin{enumerate}
    \item \textbf{Piecewise Bi-Lipschitz Property:} Since $\phi$ is a composition of linear maps and non-expansive activations, it is globally Lipschitz ($c_2 < \infty$), and its Lipschitz constant can be estimated~\cite{virmaux2018lipschitz}. While it is not injective on the entire pixel space $\mathbb{R}^D$, generative modeling concerns the restricted mapping on the \textit{supported manifold} $\mathcal{M}$. Assuming $\phi$ is a robust semantic encoder that separates semantically distinct inputs, it induces a homeomorphism onto its image $\phi(\mathcal{M})$ that is bi-Lipschitz almost everywhere. This is sufficient to preserve the Hausdorff dimension and volume scaling laws up to constant factors.
    
    \item \textbf{Semantic Volume Definition:} Crucially, our metric does not aim to recover the volume of the raw pixel manifold (which is dominated by imperceptible noise). Instead, $\Delta \varepsilon_c$ measures the volume of the \textbf{Pushforward Measure} $\phi_{\#} \mu_\pi$ in the feature space. Singularities or dimensionality reduction induced by $\phi$ are feature-preserving: they collapse semantically irrelevant variations (noise basins) into single points. Therefore, a negative shift in this space ($\Delta \varepsilon_c < 0$) rigorously implies a contraction of the \textit{semantic} repertoire, which is the exact quantity of interest in generative evaluation.
\end{enumerate}
\end{remark}
\end{document}